\newif\ifabstract
\abstracttrue
 \abstractfalse 
\newif\iffull
\ifabstract \fullfalse \else \fulltrue \fi

\documentclass[11pt,onecolumn]{article}
\usepackage{amsfonts}
\usepackage{amssymb}
\usepackage{amstext}
\usepackage{amsmath}
\usepackage{xspace}
\usepackage{theorem}
\usepackage{graphicx}
\usepackage{url}
\usepackage{graphics}
\usepackage{colordvi}
\usepackage{colordvi}
\usepackage{subfigure}

\usepackage{setspace}
\usepackage{color}
\usepackage{cite}
\usepackage[disable]{todonotes}
\usepackage{hhline}
\usepackage{array}
\usepackage{pifont}
\usepackage{enumerate}

\usepackage{graphicx,epsfig,amsfonts,bbm,epstopdf,tabularx}

\usepackage{url}

\usepackage{xcolor}
\usepackage[comma,numbers,square,sort&compress]{natbib}
\usepackage{algorithm,algpseudocode}
\usepackage{algorithmicx}
\usepackage{stmaryrd}
\usepackage{multirow}
\usepackage{graphicx}
\usepackage{arydshln}
\usepackage{mathtools}
\usepackage{tabulary}
\usepackage{booktabs}
\usepackage{subfigure}
\usepackage{xspace}

\advance \oddsidemargin by -0.8in
\newcommand{\myparskip}{3pt}
\parskip \myparskip

        {\hspace*{\fill}$\Box$\par\vspace{4mm}}

\newcommand{\be}{\begin{enumerate}}
\newcommand{\ee}{\end{enumerate}}
\newcommand{\bd}{\begin{description}}
\newcommand{\ed}{\end{description}}
\newcommand{\bi}{\begin{itemize}}
\newcommand{\ei}{\end{itemize}}

\renewcommand{\phi}{\varphi}

\newtheorem{remark}{\textbf{Remark}}

\newtheorem{lemma}{\textbf{Lemma}}
\newtheorem{theorem}{\textbf{Theorem}}
\newtheorem{corollary}{\textbf{Corollary}}

\newtheorem{proof}{Proof}

\newcommand{\AC}{\texttt{AC-CMA2B}\xspace}

\newcommand{\CMAB}{\texttt{CMA2B}\xspace}
\newcommand{\AAE}{\texttt{CO-AAE}\xspace}
\newcommand{\IAAE}{\texttt{IND-AAE}\xspace}
\newcommand{\IUCB}{\texttt{IND-UCB}\xspace}
\newcommand{\AAEbasic}{\texttt{AAE}\xspace}

\newcommand{\ucbo}{\texttt{CO-UCB}\xspace}

\newcommand{\cS}{\mathcal{K}\xspace}

\newcommand{\cint}{{\texttt{CI}}\xspace}

\def\E{\mathbb{E}}

\newcommand{\cK}{\mathcal{K}}
\newcommand{\cA}{\mathcal{A}}

\begin{document}

\title{Distributed Bandits with Heterogeneous Agents}
\author{Lin Yang\thanks{College of Information and Computer Sciences, UMass Amherst. Email: {\tt linyang@cs.umass.edu}.} \and Yu-zhen Janice Chen \thanks{College of Information and Computer Sciences, UMass Amherst. Email: {\tt yuzhenchen@cs.umass.edu}.} \and Mohammad Hajiesmaili \thanks{College of Information and Computer Sciences, UMass Amherst. Email: {\tt hajiesmaili@cs.umass.edu}.} \and John CS Lui \thanks{Department of Computer Science and Engineering, The Chinese University of Hong kong. Email: {\tt cslui@cse.cuhk.hk}.} \and Don Towsley \thanks{College of Information and Computer Sciences, UMass Amherst. Email: {\tt towsley@cs.umass.edu}.}}

\maketitle

\begin{abstract}
This paper tackles a multi-agent bandit setting where $M$ agents cooperate together to solve the same instance of a $K$-armed stochastic bandit problem. The agents are \textit{heterogeneous}: each agent has limited access to a local subset of arms and the agents are asynchronous with different gaps between decision-making rounds. 
The goal for each agent is to find its optimal local arm, and agents can cooperate by sharing their observations with others. While cooperation between agents improves the performance of learning, it comes with an additional complexity of communication between agents. 
For this heterogeneous multi-agent setting, we propose two learning algorithms, \ucbo and \AAE. 
We prove that both algorithms achieve order-optimal regret, which is $O\left(\sum_{i:\tilde{\Delta}_i>0} \log T/\tilde{\Delta}_i\right)$, where $\tilde{\Delta}_i$ is the minimum suboptimality gap between the reward mean of arm $i$ and any local optimal arm. In addition, a careful selection of the valuable information for cooperation, \AAE achieves a low communication complexity of $O(\log T)$. Last, numerical experiments verify the efficiency of both algorithms.

\end{abstract}

\section{Introduction}
\label{sec:intro}

Multi-armed bandits (MABs)~\cite{bubeck2012regret,slivkins2019introduction} fall into a well-established framework for learning under uncertainty that has been studied extensively since the 1950s after the seminal work of~\cite{robbins1952some}. MABs have a broad range of applications including online shortest path routing, online advertisement, channel allocation, and recommender systems~\cite{slivkins2019introduction, jiang2016via, langford2007epoch, bresler2016collaborative}.
In the basic MAB problem, a learner repeatedly pulls an arm in each round, and observes the reward/loss associated with the selected arm, but not those associated with others. The goal of the learner is to minimize regret, which compares the rewards/loss received by the learner to those accumulated by the best arm in hindsight.

Distributed MABs, which are extensions of basic MABs, have been studied extensively recently in different settings~\cite{liu2010distributed,szorenyi2013gossip,bistritz2018distributed,kolla2018collaborative,kalathil2014decentralized,dubey2020cooperative,martinez2019decentralized,wang2020optimal,sankararaman2019social,magesh2019multi,youssef2021resource,bande2021dynamic,magesh2021decentralized}. Distributed bandits is well motivated by a broad range application scenarios such as (1) large-scale learning systems~\cite{cesa2020cooperative}, in domains such as online advertising and recommendation systems; (2) cooperative search by multiple robots~\cite{li2014cooperative,jin2017cooperative}; (3) applications in wireless cognitive radio~\cite{bubeck2020non,liu2010distributed,liu2010decentralized,besson2018multi}; and distributed learning in geographically distributed communication systems, such as a set of IoT devices learning about the underlying environments~\cite{mcquade2012global,darak2019multi,avner2016multi,bonnefoi2017multi,xia2020multi}.
Most prior work on multi-agent MABs assume that agents are \textit{homogeneous}: all agents have full access to the set of all arms, and hence they solve the same instance of a MAB problem, with the aim to minimize the aggregate regret of the agents either in a \textit{competition} setting~\cite{anandkumar2011distributed,boursier2019sic,bubeck2020non,wang2020optimal,bistritz2018distributed,boursier2020practical,liu2010distributed,liu2010decentralized,besson2018multi}, i.e., degraded or no-reward when multiple agents pull the same arm, or in a \textit{collaboration/cooperation} setting~\cite{martinez2019decentralized,wang2020optimal,landgren2018social,landgren2016distributed,kolla2018collaborative,sankararaman2019social}, where agents pulling the same arm observe independent rewards, and agents can communicate their observations to each other in order to improve their learning performance.

\subsection{Distributed Bandits with Heterogeneous Agents}
In this paper, we study a heterogeneous version of the cooperative multi-agent MAB problem in which the agents only have partial access to the set of arms.
More formally, we study a multi-agent system with a set $\cA = \{1,\dots,M\}$ of agents and a set $\cK = \{1,\dots, K\}$ of arms. Agent $j\in\cA$ has access to a subset $\cK_j \subseteq \cK$  of arms. We refer to arms in $\cK_j$ as \textit{local} arms for agent $j$. The heterogeneity of agent also appears in their learning capabilities that lead to different \textit{action rates}; agent $j\in\cA$ can pull an arm every $1/\theta_j$ rounds, $0<\theta_j\le 1$. Here $\theta_j$ is the action rate of agent $j$.
The goal of each agent is to learn the best local arm within its local set, and agents can share information on overlapping arms in their local sets to accelerate the learning process. In this model, we assume agents are fully connected and can truthfully broadcast their observed rewards to each other. We call this setup \textit{Action-constrained Cooperative Multi-agent MAB} (\AC) and formally define it in Section~\ref{sec:model}.

\subsection{Motivating Application}
Cooperative multi-agent bandits have been well-motivated in the literature, and in the following, we motivate the heterogeneous-agent setting.
Online advertisement is a classic application that is tackled using the bandit framework. In the online advertisement, the goal is to select an ad (arm) for a product or a search query, and the reward is the revenue obtained from ads. In the context of \AC, consider a scenario that for a set of related products, a separate agent runs a bandit algorithm to select a high-reward ad for each product in the set. However, the set of available ads might have partial overlaps among multiple related products, i.e., different agents might have some overlapping arms. Hence, by leveraging the \AC model, agents running different bandit algorithms can cooperate by sharing their observations to improve their performance. One may imagine similar cooperative scenarios for recommendation systems in social networks~\cite{sankararaman2019social} where multiple learning agents in different social networks, e.g., Facebook, Instagram, cooperate to recommend posts from overlapping sets of actions. Even more broadly, the multi-agent version of classical bandit applications is a natural extension~\cite{yang2021cooperative}. For example, in online shortest path routing problem\cite{zou2014online,talebi2017stochastic}, as another classic example of bandit applications, multi-agent setting could capture the case in which the underlying network is large and each agent is responsible for routing within a sub-graph in the network. Last, it is also plausible that the to have asynchronous learning among different agents in the sense that each agent has its own action rate for decision making.


\subsection{Contributions}

The goal of this paper is to design cooperative algorithms with sublinear regret and low communication complexity.
This is challenging since these two goals can be in conflict. Intuitively, with more information exchange, the agents can benefit from empirical observations made by others, resulting in an improved regret. However, this comes at the expense of additional communication complexity due to information exchange among agents.
In this paper, we tackle \AC by developing two cooperative bandit algorithms and analyze their regret and communication complexities. The contribution is summarized as follows.

First, to characterize the regret of our algorithms, we introduce $\tilde{\Delta}_i$ as a customized notion of the suboptimality gap, which is unique to \AC.
Specifically, the parameter $\tilde{\Delta}_i$, $i\in \mathcal{K}$ (see Equation~\eqref{eq:tilde_delta} for the formal definition), measures the minimum gap between the mean reward of arm $i$ and local optimal arms of agents including $i$ in their local sets. Intuitively,  $\{\tilde{\Delta}_i\}_{i\in \cK}$ determine the ``difficulty'' of the bandit problem in a distributed and heterogeneous setting and appear in the regret bounds.

Second, we present two learning algorithms, \ucbo and \AAE, which  extend the Upper Confidence Bound algorithm and the Active Arm Elimination algorithm~\cite{even2006action} to the cooperative setting, respectively.
We use the notion of local suboptimality gap $\tilde{\Delta}_i$ and characterize the regrets of \ucbo and \AAE and show that both algorithms achieve a regrets of $O\left(\sum_{i:\tilde{\Delta}_i>0}\log T/\tilde{\Delta}_i\right)$.
By establishing a regret lower bound for \AC, we show that the above regret is optimal. To the best of our knowledge, this is the first optimality result for distributed bandits in a heterogeneous setting.
Even though both algorithms are order-optimal, the regret of \ucbo is smaller than \AAE by a constant factor (see Theorems \ref{thm:2} and \ref{thm:4}).
This is also validated by our simulations in Section~\ref{sec:exp} with real data traces.

Last, we investigate the communication complexity of both algorithms, which measures the communication overhead incurred by the agents for cooperation to accelerate the learning process.
In our work, communication complexity is defined to be the total number of messages, i.e., arm indices and observed rewards, exchanged by the agents. Our analysis shows that \ucbo generally needs to send as much as $O(M\Theta T)$ amount of messages, where $\Theta$ is the aggregate action rate of all agents. Apparently, the communication complexity of \ucbo is higher than that of \AAE, which is $O\left(\sum_{i:\tilde{\Delta}_i>0}\log T/\tilde{\Delta}^2_i\right)$.

We note that the authors in~\cite{yang2021cooperative} also tackle a cooperative bandit problem with multiple heterogeneous agents with partial access to a subset of arms and different action rates. However, in~\cite{yang2021cooperative}, the goal of each agent is to find the global optimal arm, while in this work, the goal of each agent is to find its local optimal arm. This difference leads to substantially different challenges in the algorithm design and analysis. More specifically, in~\cite{yang2021cooperative}, a foundational challenge is to find an effective cooperative strategy to resolve a dilemma between pulling local vs. external arms. This is not the case in \AC since the goal is to find the best local action. In addition, in~\cite{yang2021cooperative}, the communication complexity of algorithms is not analyzed. Our paper, instead, focuses on designing cooperative strategies with low communication complexities.

\section{Model and Preliminaries\label{sec:model}}

\subsection{System Model}
We consider a cooperative multi-agent MAB (\CMAB) setting, where there is a set $\mathcal{A} = \{1,\dots,M\}$ of independent agents, each of which has partial access to a global set $\mathcal{K} = \{1,\dots, K\}$ of arms. Let $\mathcal{K}_j \subseteq \mathcal{K}, K_j = |\mathcal{K}_j|$, be the set of arms available to agent $j\in\mathcal{A}$. Associated with arms are mutually independent sequences of i.i.d. rewards, taken to be Bernoulli with mean $0\le \mu(i)\le 1$,  $i\in \mathcal{K}$. We assume that the local sets of some agents overlap so that cooperation among agents makes sense.

In addition to differences in their access to arms, agents also differ in their decision making capabilities. Specifically, considering decision rounds $\{1,\dots,T\}$, agent $j$ can pull an arm every $\omega_j \in \mathbb{N}^+$ rounds, i.e., decision rounds for agent $j$ are $t=\omega_j,2\omega_j,\ldots,N_j\omega_j$, where ${N_j=\lfloor T/\omega_j \rfloor}$. Parameter $\omega_j$ represents the \textit{inter-round gap} of agent $j$.
For simplicity of analysis, we define $\theta_j:=1/\omega_j$ as the \textit{action rate} of agent $j$. Intuitively, the larger $\theta_j$, the faster agent $j$ can pull arms.

We assume that all agents can communicate with all other agents. Hence every time an agent pulls an arm, it can broadcast the arm index and the reward received to any other agent.  However, there is a deterministic communication delay, $d_{j_1,j_2}$ , between any two agents, $j_1$ and $j_2$, measured in units of decision rounds.

\subsection{Performance Metrics}
At each decision round, agent $j$ can pull an arm from $\cK_j$. The goal of each agent is to learn the best local arm. The regret of agent $j$ is defined as
\begin{equation}
	\label{eq:regret_ac}
	R_T^j := \mu(i_j^*) N_j - \sum\nolimits_{t\in \left\{k\omega_j:k=0,1,\ldots,N_j\right\}}x_{t}(I_{t}^{j}),
\end{equation}
where $i_j^*$ is the local optimal arm in $\mathcal{K}_j$, $I_{t}^{j} \in \cS_j$ is the action taken by agent $j$ at round $t$, and $x_{t}(I_{t}^{j})$ is the realized reward.

Without loss of generality, we assume that the local sets of at least two agents overlap, i.e., $\exists j,j'\in\mathcal{A}: \cS_j \cap \cS_{j'} \neq \emptyset$ and the overall goal is to minimize aggregate regret of all agents, i.e., $R_T=\sum_{j\in \mathcal{A}}R_T^j$.

In addition, it is costly to transmit messages in some practical networks. To measure the communication overhead of an algorithm in \AC, we simply assume that each message contains enough bits to transmit the index of an arm or an observation on the reward, and similar to \cite{wan2020projection,wang2020optimal}, the communication complexity, denoted as $C_T$, is defined to be the total number of messages sent by all agents in $[1,T]$.

\subsection{Additional Notations and Terminologies}
To facilitate our algorithm design and analysis, we introduce the following notations.
By $\mathcal{A}_{i}$, we denote set of agents that can access arm $i$, i.e.,
$\mathcal{A}_{i} :=\left\{j\in \mathcal{A}: i\in \mathcal{K}_j\right\}$.
By $\mathcal{A}_i^*$, we denote the set of agents whose optimal local arm is $i$, i.e., $\mathcal{A}^*_{i} :=\left\{j\in \mathcal{A}_i: \mu_i \geq \mu_{i'},~\text{for}~i'\in \mathcal{K}_j\right\}$.
Note that $\mathcal{A}^*_{i}$ may be empty. Moreover, let $\mathcal{A}^*_{-i} = \mathcal{A}_{i}/\mathcal{A}^*_{i}$ be the set of agents including $i$ as a suboptimal arm. Finally, let $M_i$, $M_{i^*}$, and $M_{-i}$ be the sizes of $\mathcal{A}_{i}$, $\mathcal{A}_{i^*}$, and $\mathcal{A}_{-i}$ respectively.



By $\Delta(i,i')$, we denote the difference in the mean rewards of arms $i$ and $i'$, i.e., $\Delta(i,i') :=\mu(i)-\mu(i')$.
Specifically, $\Delta(i^*,i)$ written as $\Delta_i$, which is known as the sumpoptimality gap in the basic bandit problem.
In additional to this standard definition, we introduce a \CMAB-specific version of the suboptimality gap, denoted by $\tilde{\Delta}_i$ as follows
\begin{equation}
	\label{eq:tilde_delta}
	\tilde{\Delta}_i:= \left\{\begin{array}{ll}\min_{j\in \mathcal{A}_{-i}} \Delta(i_j^*,i),& \mathcal{A}_{-i}\neq \emptyset; \\
		0,& \text{otherwise}.\end{array}\right.
\end{equation}
Last, we define $\Theta$ and $\Theta_i$, $i\in \mathcal{K}$, as follows.
\begin{equation*}
	\begin{split}
		\Theta :=\sum_{j\in \mathcal{A}} \theta_j, \quad \text{and} \quad \Theta_{i} :=\sum_{j\in \mathcal{A}_i} \theta_j.
	\end{split}
\end{equation*}
\begin{remark}
Both $\tilde{\Delta}_i$ and $\Theta_{i}$ play key roles in characterizing the regret bounds of an algorithm in \AC. Specifically, $\tilde{\Delta}_i$ measures the minimum gap of the reward mean between the local optimal arm of agents in $\mathcal{A}_{-i}$ and arm $i$, and $\Theta_{i}$ measures the aggregate action rate of agents in $\mathcal{A}_i$, and roughly the larger the $\Theta_{i}$, the higher the rate at which arm $i$ could be pulled by the set of agents that belongs to, i.e., $\mathcal{A}_i$.
\end{remark}

\section{Algorithms}\label{sec:algorithm}

In \AC, each agent has to identify the optimal local optimal arm and the learning process can be improved by communicating with the other agents with common arms. 
The traditional challenge for MAB comes from the exploration-exploitation dilemma. In \AC, the agents have to resolve this by designing learning algorithms with low regret and low communication complexity. 
The heterogeneity in action rates and access to the decision set exacerbates the design and analysis of cooperative learning algorithms for \AC. In this section, we present two algorithms: \ucbo and \AAE. \ucbo generalizes the classic Upper Confidence Bound algorithm that achieves good regret but incurs high communication complexity.
\AAE borrows the idea of the arm elimination strategy but incorporates a novel communication strategy tailored to reduce communication complexity. In Section~\ref{sec:regret}, we derive a regret lower bound for \AC, analyze regrets and communication complexities for both algorithms, show the optimality of the regrets for both algorithms, and show that \AAE achieves low communication complexity.   

\subsection{Confidence Interval} 
Both \ucbo and \AAE use confidence of the mean rewards to make decisions. A simple introduction on the confidence interval is provided below. 
In \AC, each agent can compute empirical mean rewards of the arms. For arm $i \in \cS$ with $n$ observations, the mean reward is denoted as $\hat{\mu}(i,n)$\footnote{In the algorithm pseudocode, we drop $t$ and $j$ from the notations $\hat{n}^j_{t}(i)$ and ${\hat{\mu}(i,\hat{n}^j_t(i))}$ for brevity, and simplify them as $\hat{n}(i)$ and $\hat{\mu}(i)$, respectively. The precise notation, however, is used in analysis.}, which is the average of the $n$ observations on arm $i$.
With these observations, we can compute a confidence interval for the true mean reward.
Specifically, the confidence interval for arm $i$ and agent $j$ at time $t$ centers its estimated mean value, $\hat{\mu}(i,n)$, and its width is defined as
\begin{equation}
\label{eq:cint}
    \cint(i,j,t) := \sqrt{\frac{\alpha \log \delta_t^{-1}}{2\hat{n}^j_{t}(i)}},
\end{equation}
where $\hat{n}^j_{t}(i)$ is the total number of observations (including both local observations and those received from other agents) of arm $i$ available to agent $j$ by time $t$ (observations made in time slots from 1 to $t-1$).
Here $\delta_t > 0$ and $\alpha >2$ are parameters of the confidence interval.
Consequently, we build the following confidence interval for arm $i \in \cS$:
\begin{equation*}\label{eq:con}
\begin{split}
\mu(i) \in \left[\hat{\mu}(i,\hat{n}_t^j(i))-\cint(i,j,t), \hat{\mu}(i,\hat{n}_t^j(i))+\texttt{CI}(i,j,t)\right],
\end{split}
\end{equation*}
where $\mu(i)$ satisfies the upper (or lower) bound with probability at least $1-\delta_t^\alpha$ ($ 0 < \delta_t  \leq 1$ is a specified parameter at time slot $t$). 
One can refer to \cite{bubeck2012regret} for a detailed analysis of the above confidence interval.

\subsection{\ucbo: Cooperative Upper Confidence Bound Algorithm}
\label{subsec:coucb}

In this subsection, we present \ucbo, a cooperative bandit algorithm for the \AC model. According to \ucbo, each agent selects the arm with the largest upper confidence bound. For agent $j$, there is
\begin{equation*}
    I^j_t=\max_{i\in \mathcal{K}} \hat{\mu}\left(i,\hat{n}_t^j(i)\right)+\cint(i,j,t). 
\end{equation*}

With each observation received from the selected arm or other agents, \ucbo updates the mean reward estimate and the upper confidence bound. In the meantime, each observation received from local arms are broadcast to other agents that contain the corresponding arm in their local sets. Details of \ucbo are summarized in Algorithm \ref{alg:1}.

\subsection{\AAE: Cooperative Active Arm Elimination Algorithm}
\label{subsec:coaae}

\begin{algorithm}[t]
	\caption{The \ucbo Algorithm for Agent $j$ }
	\label{alg:1}
	\footnotesize
	\begin{algorithmic}[1]
		\State \textbf{Initialization:} $\hat{n}(i)=0$, $\hat{\mu}(i)$, $i\in \mathcal{K}_j$; $\alpha>2$, $\delta_t$.
		
		\For{each ecision round $t=l/\theta_j$ ($l \in \{1,\dots,N_j\}$)}
		\State Pull arm $I_t^j$ with the highest upper confidence bound
		\State Increase $\hat{n}(I_t^j)$ by 1
		\State Update the empirical mean value of $\hat{\mu}(I_t^j)$
		\State Broadcast $x_t(I_t^j)$ to other agents which contains arm $I_t^j$ \label{algline:broadcast}
		\EndFor
		\For{each newly received $x_t(i)$, $i\in \mathcal{K}_j$ from the past decision round} \label{algline:update}
		\State Execute Lines (4)-(5)  \label{algline:update2}
		\EndFor
	\end{algorithmic}
\end{algorithm}

\begin{algorithm}[t]
	\caption{The \AAE Algorithm for Agent $j$ }
	\label{alg:2}
	\footnotesize
	\begin{algorithmic}[1]
		\State \textbf{Initialization:} $\hat{n}(i)=0$, $\hat{\mu}(i)$, $i\in \mathcal{K}_j$; $\alpha>2$, $\delta_t$.
		\For{each received $x_\tau(i)$, $\tau<t$, $i\in \mathcal{K}_j$ for past rounds}
		\State Execute Lines (\ref{algline:update_beg})-(\ref{algline:update_end})  
		\EndFor
		\For{each decision round $t=l/\theta_j$ ($l \in \{1,\dots,N_j\}$)}
		\State Pull arm $I_t^j$ from the candidate set as constructed in Equation~\eqref{eq:cand_set} with the least observations
		\State Increase $\hat{n}(I_t^j)$ by 1 and update the empirical mean value, $\hat{\mu}(I_t^j)$ \label{algline:update_beg}
		\State Reconstruct the candidate set based on the updated values of $\hat{n}(I_t)$  and $\hat{\mu}(I_t^j)$ by using Equation~\eqref{eq:cand_set} 
		
		\If{one arm is eliminated}
		\State Broadcast the indices of eliminated arms to other agents 
		\EndIf\label{algline:update_end}
		\If{the candidate set contains more than 1 arms}
		\State Broadcast $x_t(I_t^j)$ to other agents whose candidate set contains arm $I_t^j$ and has more than one arms
		\EndIf
		\EndFor
	\end{algorithmic}
\end{algorithm}

\AAE is independently executed by each agent and is summarized as Algorithm \ref{alg:2}. By maintaining the confidence intervals of local arms, \AAE maintains a \textit{candidate set} to track the arms likely to be the optimal local arm. The candidate set is initially the entire local set, and when the confidence interval of an arm comes to lie below that of another arm, the arm is removed from the candidate set. During execution, \AAE selects the arm with the fewest observations from the candidate set. 
With the introduction of the candidate set, \AAE avoids sending messages regarding low-reward arms resulting in a lower communication complexity than \ucbo . Details are introduced below.

\paragraph{Selection Policy for Local Arms}
Now we present details on constructing the candidate set for agent $j$. The candidate set of $j$ originally contains all arms in $\mathcal{K}_j$. Then,  \AAE eliminates those arms whose confidence intervals lie below those of other arms without further consideration, and keeps the rest in a dynamic candidate set of arms. Specifically, the candidate set $\mathcal{C}_{j,t}$ is defined in (\ref{eq:cand_set}).
\begin{equation}
	\label{eq:cand_set}
\mathcal{C}_{j,t}:=\left\{i\in \mathcal{K}_{j}:\hat{\mu}(i,\hat{n}^j_{t}(i))+\texttt{cint}(i,j,t)\geq \hat{\mu}(i',\hat{n}^j_{t}(i))-\texttt{cint}(i,j,t), ~\text{for}~\text{any}~i'\in \mathcal{K}_{j}\right\}.
\end{equation}
The agent updates the candidate set after pulling an arm and each time it receives an observation from another agent. Note that communication delays and agent action rates are heterogeneous. Hence, the recorded number of observations and empirical mean rewards vary among agents.
To balance the number of observations among different local arms, the agent at each time slot pulls the arm within its local candidate set with the least number of observations. 

\paragraph{Communication Policy}
In order to reduce the communication complexity, it is also crucial for \AAE to decide how to share information among different agents. 
During the execution of \AAE, each agent updates its candidate set with its received observations. 
When an arm is eliminated from an agent's candidate set, the agent will broadcast the index of the eliminated arm, such that all agents can track the candidate sets in others. 
In the following, we will introduce a novel communication policy tailored for the \AAE algorithm. 
The communication policy of \AAE generally follows the following two rules.

\begin{enumerate}
\item An agent only broadcasts observations to the agents whose candidate set has more than one arms and contains the corresponding arms.  
\item When there is only one arm in the candidate set, the agent will also broadcast any observations to other agents.
\end{enumerate}

By the first rule, the communication policy can avoid transmitting redundant observations to the agents which have finished the learning task or have only one arm in their candidate sets. 
In addition, the second rule can prevent the ``fast'' agents which quickly eliminated suboptimal arms from sending too many observations to the ``slow''  agents containing the arms whose means are close to the local optimal arm.
Otherwise, sending too many observations on those arms to ``slow''  agents may incur $O(T)$ communication complexity in the extreme case.

\begin{remark}
The above communication policy can be easily implemented in practical systems and work efficiently in a fully distributed environment without knowing the parameters of other agents, such as action rates.  Previous communication policies in distributed bandits, such as those in \cite{wang2020optimal,martinez2019decentralized} etc., require a centralized coordinator, which is difficult to implement in our heterogeneous setting.
We also note that, the above communication policy is not applicable to \ucbo, since each agent fails to send out explicit signals on suboptimal arms.
\end{remark}
\section{Theoretical Results}
\label{sec:regret}

In this section, we present our results for \ucbo and \AAE, respectively.
For \AC, a theoretical challenge is to characterize the regret bounds with respect to the constraint that agents can only pull arms from predetermined and possibly overlapping sets of arms.
This challenge can be tackled by incorporating the agent-specific suboptimality gaps introduced in \eqref{eq:tilde_delta} into the regret analysis.
In this section, we provide upper and lower bounds for the regrets in the \AC setting, which all depend on the agent-specific suboptimality gaps $\tilde{\Delta}_i$. Also, in this section, by policy, we mean the way that each agent determines which arm should be selected in each decision round.  Proofs are given in Section~\ref{sec:proofs}.

\subsection{An Overview of Our Results}
Let $\textsf{KL}(u,v)$ be the Kullback-Leibler divergence between a Bernoulli distribution with parameters of $u$ and $v$, i.e.,
$\textsf{KL}(u,v)= u\log (u/v)+(1-u)\log ((1-u)/(1-v)).$

\begin{theorem}\label{thm:lb}
(Regret Lower Bound for \AC) Consider a case where $\theta_j=O(1)$ for $j\in \mathcal{A}$ and a policy that satisfies ${\E\left[n_T(i)\right]=o(T^a)}$ for any set of Bernoulli reward distributions, any arm $i$ with $\tilde{\Delta}_i>0$, and any $a>0$.
 Then, for any set of Bernoulli reward distributions, the expected regret for any algorithm satisfies
	\begin{equation*}
		{\lim\inf}_{T\rightarrow \infty} \frac{\E\left[R_T\right]}{\log T} \geq \sum_{i:\tilde{\Delta}_i>0} \frac{\tilde{\Delta}_i}{\textsf{KL}(\mu_i,\mu_i+\tilde{\Delta}_i)},
	\end{equation*}
\end{theorem}
The proof leverages the similar techniques of the classical result for the basic stochastic bandits~\cite{lai1985asymptotically} and is given in Subsection~\ref{subsec:lower_bound}. We proceed to introduce the following notations to facilitate the presentation of the regret bounds for both algorithms. We define
\begin{equation*}
    q_1:=2\sum_{j\in \mathcal{A}}\sum_{l=1}^{N_j}\sum_{i\in \mathcal{K}_j}\frac{l\Theta_i}{\theta_j} \delta_{l/\theta_j}^{\alpha},
\end{equation*}
\begin{equation*}
f_i(\delta):=\sum_{j\in \mathcal{A}_i}\min\left\{d_j\theta_j,\frac{2\alpha\log \delta^{-1}}{\Delta^2(i_{j}^*,i)}\right\},
\end{equation*}
where $\alpha>2$, $\delta_{l/\theta_j}$ are parameters specified by the proposed algorithms, and $\delta:=\max_l \delta_{l/\theta_j}$.

\begin{theorem}\label{thm:2}
	(Expected regret of \ucbo) With $\alpha>2$, the expected regret of the \AAE algorithm satisfies
	\begin{equation*}
	\begin{split}
		\E\left[R_T\right]   \leq\sum_{i:\tilde{\Delta}_i>0}\left(\frac{6\alpha\log \delta^{-1}}{\tilde{\Delta}_i}+q_1+f_i(\delta)\right).
	\end{split}
	\end{equation*}
\end{theorem}

By setting $\delta_t=1/t$, we have
\begin{equation*}
\begin{split}
&2\sum_{j\in \mathcal{A}}\sum_{l=1}^{N_j}\sum_{i\in \mathcal{K}_j}\frac{l\Theta_i}{\theta_j} \delta_{l/\theta_j}^{\alpha}=2\sum_{j\in \mathcal{A}}\sum_{l=1}^{N_j}\sum_{i\in \mathcal{K}_j}\Theta_i \frac{1}{(l/\theta_j)^{\alpha-1}} \\
\leq & 2\sum_{j\in \mathcal{A}}\sum_{l=1}^{N_j}\Theta \frac{1}{(l/\theta_j)^{\alpha-1}}
\leq  \frac{2}{\alpha-2}\sum_{j\in \mathcal{A}} \Theta\theta_j^{\alpha-1}. \\
\end{split}
\end{equation*}

We define \begin{equation*}
    q_2:=\frac{2}{\alpha-2}\sum_{j\in \mathcal{A}} \Theta\theta_j^{\alpha-1}.
\end{equation*}

Applying the above results and definitions to Theorem \ref{thm:2} yields the following corollary.

\begin{corollary}\label{cor:1}
	With $\delta_t=1/t$ and $\alpha>2$, the \ucbo algorithm attains the following expected regret
	\begin{equation*}
	\begin{split}
		\mathbb{E}\left[R_T\right]
		\leq  \sum_{i:\tilde{\Delta}_i>0}\left(\frac{6\alpha\log T}{\tilde{\Delta}_i}+f_i\left(\frac{1}{T}\right)+1\right)+q_2.
	\end{split}
	\end{equation*}
\end{corollary}
In addition to regret, we are also interested in the communication complexity of \ucbo. For simplicity, we assume that one message is needed to send an observation from an agent to another one.
The total number observations made by all agents is $\Theta T$. Then, broadcasting an observation on arm $i$ to all other agents incurs at most $M$ communications complexity. Hence, the total communication complexity of \ucbo is $O(M\Theta T)$, which is formally summarized in the following theorem.

\begin{theorem} \label{thm:comlexity1}(Communication complexity of \ucbo)
The communication complexity of \ucbo is $O(M\Theta T)$.
\end{theorem}
Now, we proceed to present the regret and communication complexity of \AAE.
Similarly, we define
\begin{equation*}
g_i(\delta):=\sum_{j\in \mathcal{A}_i}\min\left\{d_j\theta_j,\frac{8\alpha\log \delta^{-1}}{\Delta^2(i_{j}^*,i)}\right\}.
\end{equation*}
We have the following theorem and corollary showing the expected regret of \AAE.

\begin{theorem}\label{thm:4}
	(Expected regret for \AAE) With $\alpha>2$, the expected regret of the \AAE algorithm satisfies
	\begin{equation*}
		\E\left[R_T\right]\leq \sum_{i:\tilde{\Delta}_i>0}\left(\frac{24\alpha\log \delta^{-1}}{\tilde{\Delta}_i}+q_1+g_i(\delta)+1\right).
	\end{equation*}
\end{theorem}

\begin{corollary}\label{cor:2}
	With $\alpha>2$ and $\delta_t=1/t$, the \AAE algorithm attains the following expected regret
	\begin{equation*}
	\begin{split}
		\mathbb{E}\left[R_T\right]
		\leq \sum_{i:\tilde{\Delta}_i>0}\left(\frac{24\alpha\log T}{\tilde{\Delta}_i}+g_i\left(\frac{1}{T}\right)+1\right)+q_2.
	\end{split}
	\end{equation*}
\end{corollary}

\begin{theorem} \label{thm:comlexity2}(Communication complexity of \AAE) Let $\delta_t=1/t$ and $\alpha>2$.
The communication complexity of \AAE satisfies
	\begin{equation*}
	C_T\leq	\sum_{i\in \mathcal{K}}\left(\frac{8\alpha\log T}{\tilde{\Delta}^2_i}+\sum_{j\in \mathcal{A}^*_{-i}}d_j\theta_j+q_2+1\right)(M+M_i).
	\end{equation*}
\end{theorem}

\subsection{Discussions}

\paragraph{Regret Optimality of Proposed Algorithms}
The first observation regarding Corollary~\ref{cor:1} and \ref{cor:2} is that the regrets linearly depend on the delay when it is not too large.
Generally, $f_i(1/T)$ and $g_i(1/T)$ relate to the number of outstanding observations that have not yet arrived.
Considering the fact that $\theta_j\leq 1$, and $\text{KL}(\mu_i,\mu_i+\tilde{\Delta}_i)$ satisfies
\begin{equation}\label{eq:bd_kl}
2\tilde{\Delta}^2_i\leq \text{KL}(\mu_i,\mu_i+\tilde{\Delta}_i) \leq \frac{\tilde{\Delta}^2_i}{(\mu(i)+\tilde{\Delta}_i)(1-\mu(i)-\tilde{\Delta}_i)},
\end{equation}
one can easily observe that both regrets match the regret lower bound when delays are bounded by a constant.

\paragraph{Comparison with Policies without Cooperation} Without cooperation, we can build up a lower bound for the regret of each agent $j$ by Theorem 2.2 in \cite{bubeck2012regret}, that is
\begin{equation*}
	{\lim\inf}_{T\rightarrow \infty} \frac{\E\left[R_T^j\right]}{\log T} \geq \sum_{i\in \mathcal{K}_j:\Delta(i_j^*,i)>0} \frac{\Delta(i_j^*,i)}{\text{KL}(\mu_i,\mu_i+\Delta(i_j^*,i))}.
\end{equation*}
Combined with Equation (\ref{eq:bd_kl}), the best regret that any non-cooperative algorithm can achieve for the integrated system is not better than
\begin{equation*}
\sum_{j\in \mathcal{A}}\sum_{i:i\in \mathcal{K}_j,\Delta(i_j^*,i)>0}\frac{\log T}{\Delta(i_j^*,i)}=\sum_{i\in \mathcal{K}}\sum_{j\in \mathcal{A}_i/\mathcal{A}^*_i}\frac{\log T}{\Delta(i_j^*,i)}.
\end{equation*}

Note that, with bounded delays, the regret upper bound of either \ucbo or \AAE is $O\left(\sum_{i:\tilde{\Delta}_i>0} \log T/\tilde{\Delta}_i\right)$.
By the definition of $\tilde{\Delta}_i$ in Equation~\eqref{eq:tilde_delta}, we have
\begin{equation*}
O\left(\sum_{i\in \mathcal{K}}\sum_{j\in \mathcal{A}_i/\mathcal{A}^*_i}\frac{\log T}{\Delta(i_j^*,i)}\right)\geq O\left(\sum_{i:\tilde{\Delta}_i>0}\frac{\log T}{\tilde{\Delta}_i}\right).
\end{equation*}

To conclude, a non-cooperative strategy will have a much larger regret than \ucbo and \AAE, especially when the number of agents is large.

\paragraph{Comparison in a Special Case without Action Constraints}
	Theorem~\ref{thm:2} and \ref{thm:4} show that the regret upper bounds depend on the new suboptimality parameter $\tilde{\Delta}_i$, which measures the minimum gap between arm $i$ and the local optimal arm.
	Intuitively, the closer the expected reward of the local optimal arm to that of the global optimal arm, the smaller the regret will be.
	Specifically, in the special case where each agent can access the global arm set and delays are bounded, we have $\tilde{\Delta}_i=\Delta_i$ and thus the expected regret of either \ucbo or \AAE becomes	$O\left(\sum_{i\in \mathcal{K}}(\alpha\log T)/\Delta_{i}\right)$.
	In the basic bandit model, a learning algorithm suffers a similar regret lower bound that depends on $\Delta_i$, i.e.,
\begin{equation*}
{\lim\inf}_{T\rightarrow \infty} \frac{\E\left[R_T\right]}{\log T} \geq \sum_{i:{\Delta}_i>0} \frac{{\Delta}_i}{\text{KL}(\mu_i,\mu_i+{\Delta}_i)}.
\end{equation*}
Thus, by assuming a constant delay and $\Theta$, the regret matches the lower bound in the special case.

\paragraph{Performance with Large Delays} We take \AAE as an example. In the extreme case where the maximum delay is arbitrarily large, the regret bound given in Corollary \ref{cor:1} becomes
\begin{equation*}
\begin{split}
	\mathbb{E}\left[R_T\right]\leq&\sum_{i:\tilde{\Delta}_i>0}\left(\frac{16\alpha\log T}{\tilde{\Delta}_i}+\sum_{j'\in \mathcal{A}_i}\frac{8\alpha\log T}{\Delta(i_{j'}^*,i)}\right)+q_2+K\\
=&\sum_{i:\tilde{\Delta}_i>0}\frac{16\alpha\log T}{\tilde{\Delta}_i}+\sum_{j\in \mathcal{A}}\sum_{i\in \mathcal{K}_j}\frac{8\alpha\log T}{\Delta(i_{j}^*,i)}+q_2+K,
\end{split}
\end{equation*}
where the second term dominates and the above regret matches  that of non-cooperative learning algorithms.

\paragraph{Communication Complexity}
The regrets of \ucbo and \AAE both drop the heavy dependency on the number of agents, but they incur much different communication overheads.
By our results, \AAE possesses much lower communication complexity than \ucbo, which is $O\left(\sum_{i\in \mathcal{K}}(M\alpha\log T)/\tilde{\Delta}^2_i\right)$. We leave it as an open problem to design the algorithm which simultaneously attains the lowest communication complexity and the regret independent of the number of agents.

\section{Numerical Experiments}\label{sec:exp}
In this section, we illustrate the performance of our proposed algorithms for the \AC settings through numerical experiments. For \AC, our goal is to evaluate the performance of \ucbo and \AAE, including regret and communication complexity, and compare them to that of  non-cooperative algorithms where each agent uses only its local observations to find the best arm. Then, we investigate the impact of communication delay on the performance of proposed algorithms in \AC.

\subsection{Overview of Setup}
We assume there are $K=100$ arms with Bernoulli rewards with average rewards uniformly randomly taken from \textit{Ad-Clicks} \cite{Kaggle2015}. 
In experiments, we report the cumulative regret after 30,000 rounds, which corresponds to the number of decision rounds of the fastest agent. All reported values are averaged over 10 independent trials and standard deviations are plotted as shaded areas. The allocation of arms to agents and number of agents differ in each experiment as explained in the corresponding sections.

\subsection{Experimental Results}
\begin{figure*}
\centering
\begin{subfigure}[Cumulative regrets with different number of agents.]{
	\centering
	\includegraphics[width=0.3\textwidth]{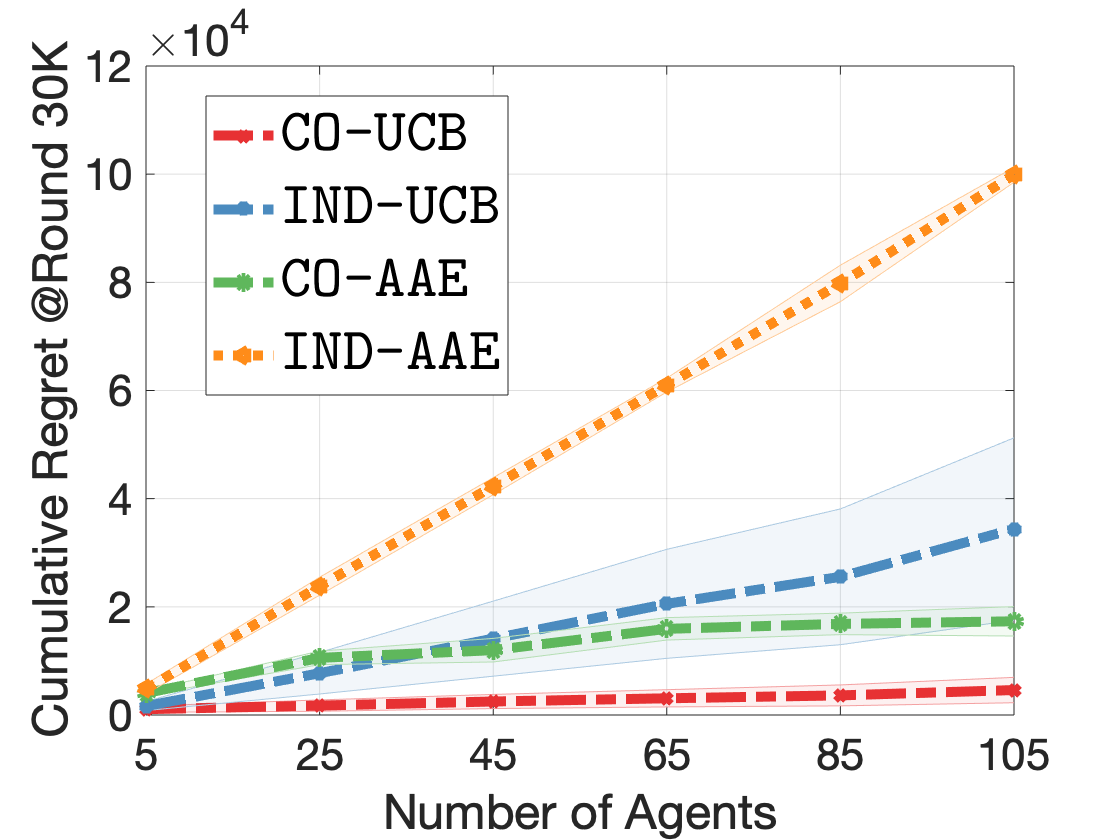}
	\label{fig:exp_co_aae_2}}
\end{subfigure}
\begin{subfigure}[Average per-agent regret with different numbers of agents.]{
\centering
\includegraphics[width=0.3\textwidth]{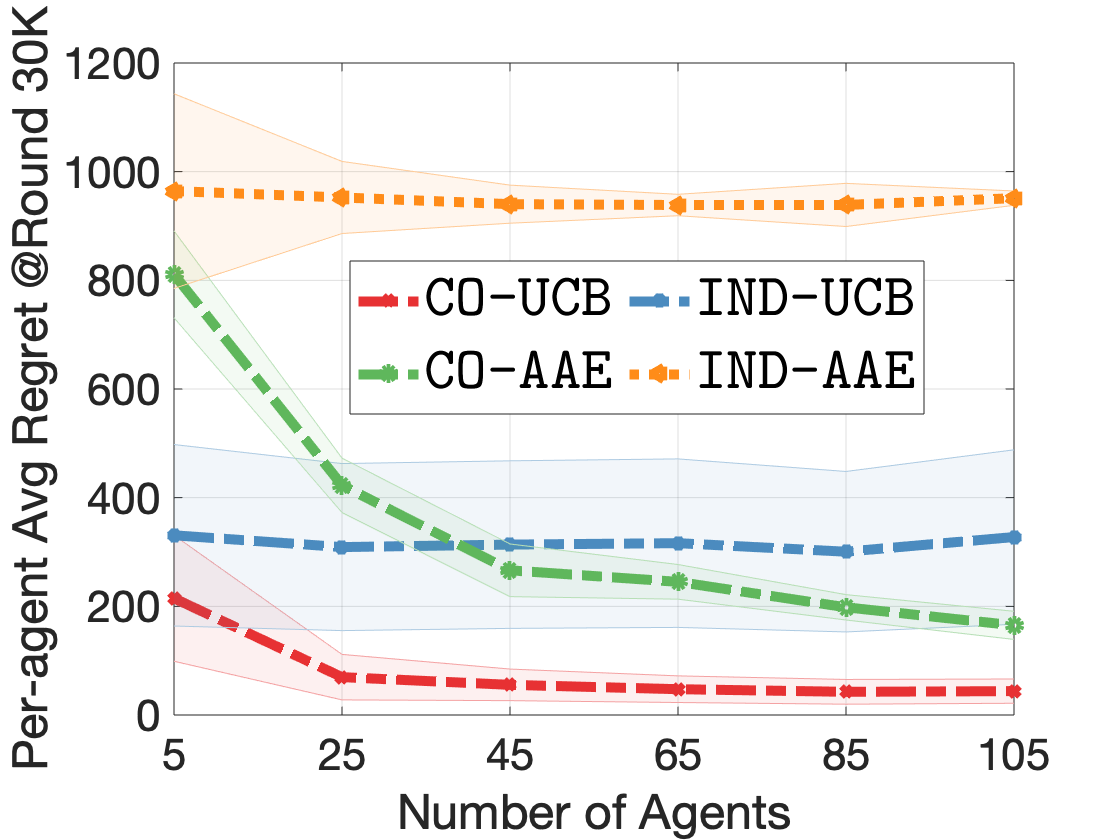}\label{fig:exp_co_aae_1}}
\end{subfigure}
\begin{subfigure}[Communication overhead.]{
\centering
\includegraphics[width=0.3\textwidth]{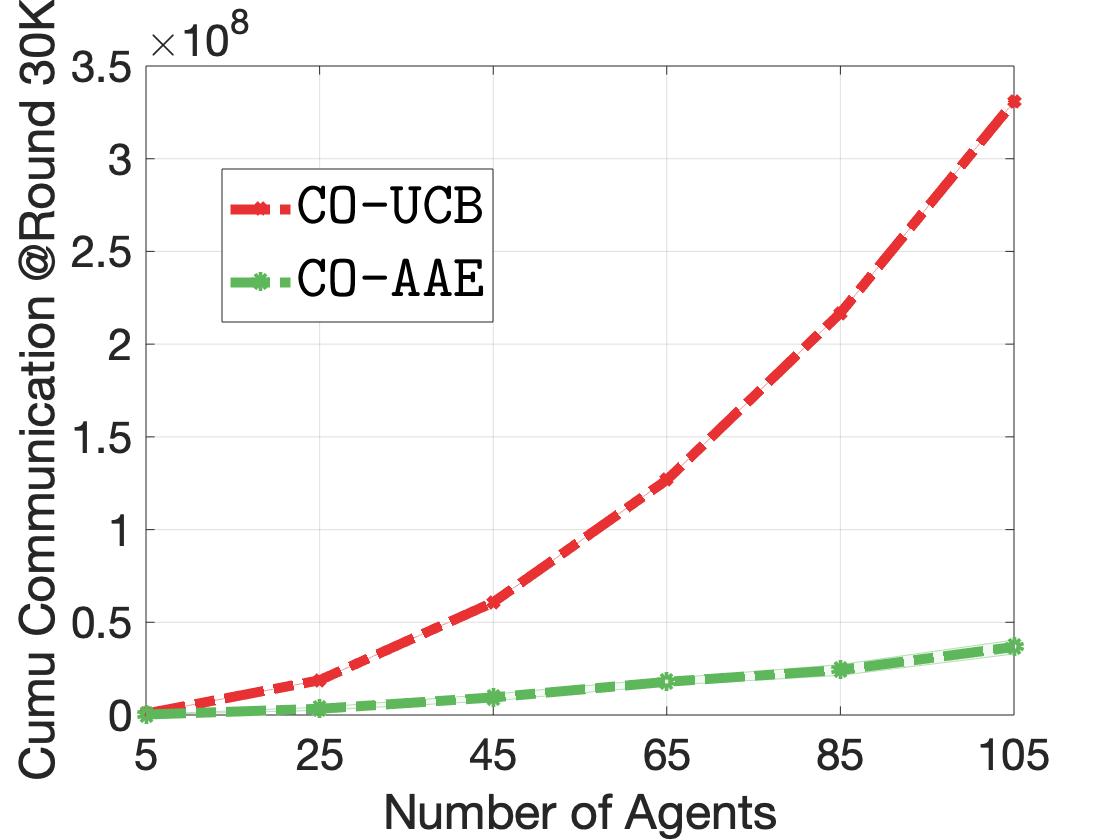}\label{fig:comm}}
\end{subfigure}
\caption{Simulation results for \AC with different number of agents in the system.}\label{fig:exp_1}
\end{figure*}

\begin{figure*}
\centering
\begin{subfigure}[Per-agent regret by \ucbo and \IUCB.]{
\centering
 \includegraphics[width=0.4\textwidth]{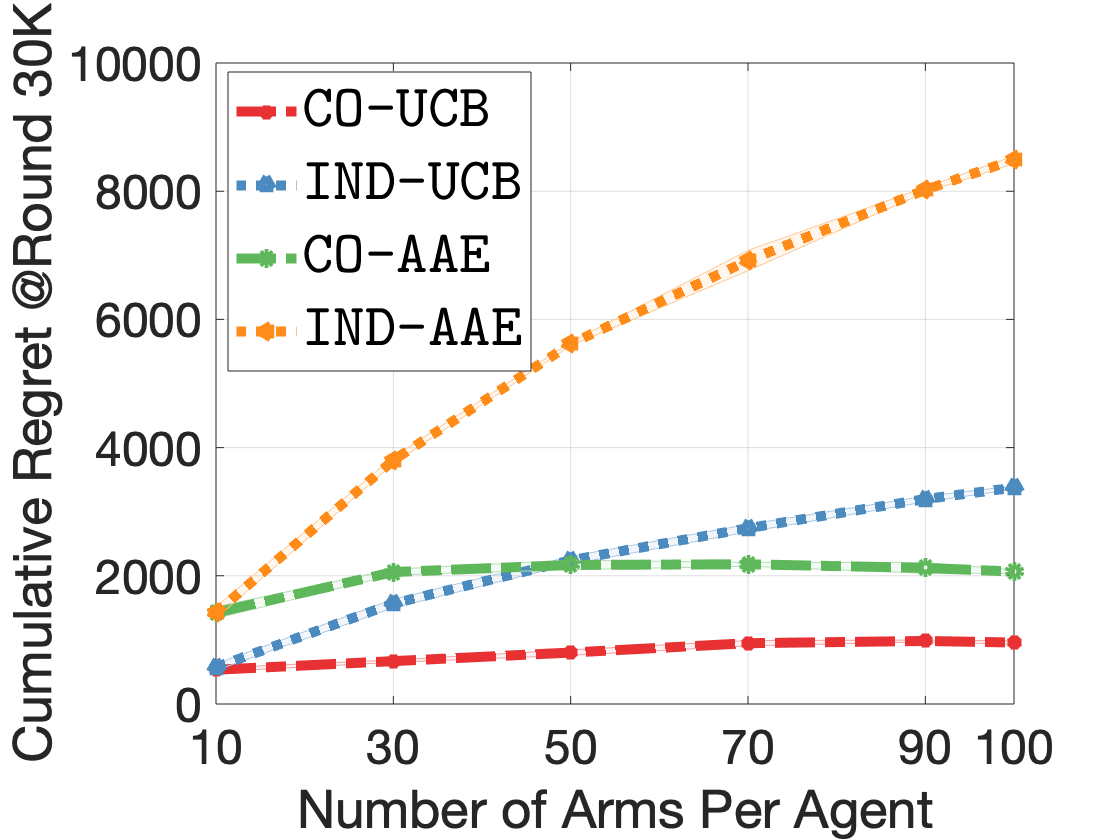}\label{fig:exp_co_aae_3}}
\end{subfigure}
\begin{subfigure}[Performance of \AAE with different delays.]{
\centering
\includegraphics[width=0.4\textwidth]{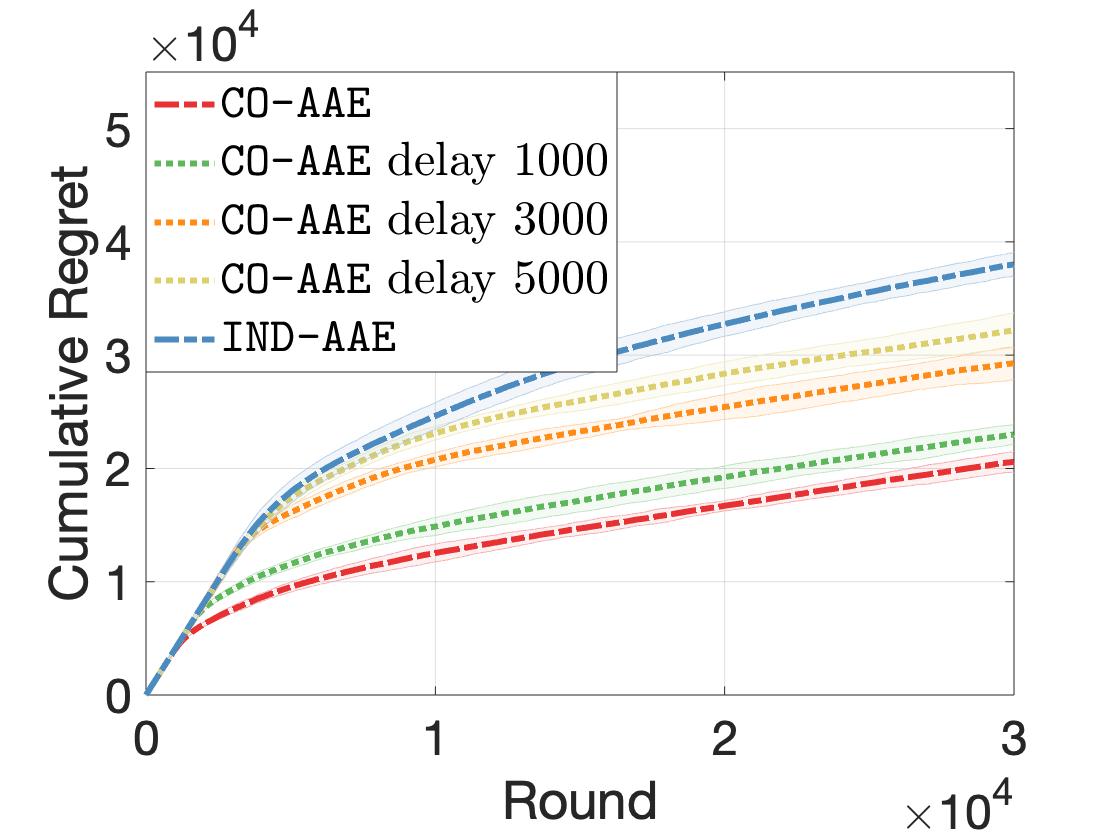}\label{fig:exp_co_aae_4}}
\end{subfigure}
\caption{Simulation results for \AC with different number of arms in each agent.}
\end{figure*}

\paragraph{Experiment 1} In the first experiment, we fix the total number of arms to $K=20$, and fix the number of arms per agent to $|\cS_j| = 6, j\in\mathcal{A}$.  We further vary the number of agents from $M=5$ (light overlap) to $M=105$ (heavy overlap), with step size of 20. 

The results are shown in Figure~\ref{fig:exp_1}. We see the observation of better performance of \ucbo and \AAE as compared to \IUCB and \IAAE. Figure~\ref{fig:exp_co_aae_2} shows a rapid increase in the cumulative regret of non-cooperative algorithms, while that of the cooperative algorithms remains the same despite the increase in the number of agents when the number of agents is larger than 65.
Figure~\ref{fig:exp_co_aae_1} depicts almost no change in the average per-agent regret of \IUCB and \IAAE, and a significant decrease for that of \ucbo and \AAE, that is due to greater overlap in the local arm sets. 

Figure~\ref{fig:comm} shows an increase of communication overheads for both \ucbo and \AAE. 
Specifically, the \AAE algorithm incurs much lower communication overhead than \ucbo in all experiments, validating our results in Theorem \ref{thm:comlexity1} and \ref{thm:comlexity2}. In addition, another important observation is that, with more agents, the communication overheads for both algorithms increase more and more quickly. That is because, when there are more agents, there will be more possibility for agents to cooperate, with more observations exchanged on common arms.

\paragraph{Experiment 2} In the second experiment, we set $K=100$ arms, and $M=10$ agents, and vary the number of arms in each agent $j$ from $|\cS_j| = 10$ with no overlap, to four partially overlapped cases, i.e., $|\cS_j| = \{30, 50, 70, 90\}$, to $|\cS_j| = 100$ which represents the complete overlap. The cumulative regret at 30000 rounds for five cases are reported in Figure~\ref{fig:exp_co_aae_3}. Figure~\ref{fig:exp_co_aae_3} shows that cooperative algorithms significantly outperform non-cooperative algorithms in general cases. One can see the gap between the performance of cooperative and non-cooperative algorithms increases as the overlap increases. This observation depicts that \ucbo and \AAE benefits from cooperation. 

\paragraph{Experiment 3} Last, we investigate the performance of the cooperative algorithm with different delays and take \AAE as an example. Toward this, we consider three additional scenarios with average delays of 1000, 3000 and 5000 slots. 
At each time slot, the exact delay is taken uniformly randomly in a given region. In Figure~\ref{fig:exp_co_aae_4}, we report the evolution of cumulative regret of \AAE. The results show that the regret of \AAE for \AC increases and approaches the regret of \IAAE as the delay increases.
\section{Conclusion}

In this paper, we study the cooperative stochastic bandit problem with heterogeneous agents, with two algorithms, \ucbo and \AAE proposed. Both algorithms attain the optimal regret, which is independent of the number of agents.
However, \AAE outperforms \ucbo in communication complexity: \AAE needs to send $O\left(\sum_{i\in \mathcal{K}}(M\alpha\log T)/(\tilde{\Delta}^2_i)\right)$ amount of messages, while the communication complexity of \ucbo is $O(M\Theta T)$. 
This paper also motivates several open questions. A promising and practically relevant work is to design the algorithm which simultaneously attains the lowest communication complexity and the regret independent of the number of agents.


\small
\bibliographystyle{abbrv}
\bibliography{main}

\appendix
\section{Proofs}\label{sec:proofs}
To facilitate our analysis, we define $n^j_t(i)$ as the number of times agent $j$ observes arm $i$, $i\in \mathcal{A}_j$, up to time slot $t$, and $n_t(i)$ is the total number of observations on arm $i$ by pulling it locally. Indeed, we have $n_{t}(i)=\sum_{j}n^j_{t}(i)$. Note that in the algorithm design, we also define $\hat{n}^j_t(i)$ as the total number of observations of arm $i$ for agent $j$ either by pulling locally, or by receiving the information from other agents.

\subsection{A Proof of Theorem~ \ref{thm:lb}}
\label{subsec:lower_bound}

The techniques for the proof of the lower bound in the basic setting have been investigated extensively and can be applied to \AC by slight modification.
For the completion of analysis, we provide the details as follows.
Let us define $\mathcal{E}_K$ as the class of $K$-armed \AC where each arm has a Bernoulli reward distribution. Assume that policy $\pi$ is consistent over $\mathcal{E}_K$, i.e., for any bandit problem $\nu\in \mathcal{E}_K$ and any $\sigma>0$, whose regret satisfies
\begin{equation*}
    R_{T}(\pi,\nu)=O((T\Theta)^{\sigma}),~\text{as}~T\rightarrow+\infty.
\end{equation*}

Let $\nu=[P_1,P_2,\ldots,P_K]$ and $\nu'=[P'_1,P'_2,\ldots,P'_K]$ be two reward distributions such that $P_k=P'_k$ except for $k=i$, where $i$ is a sub-optimal arm for all agents.
Specifically, we choose $P'_{i}=\mathcal{N}(\mu_{i}+\lambda)$ and $\lambda>\Delta_{i}$.
For stochastic bandits, we have the following divergence decomposition equation (one can refer to \cite{basu19privacy} for more details).
\begin{equation*}
    \text{KL}(\mathbb{P}_{\nu,\pi},\mathbb{P}_{\nu',\pi})=\mathbb{E}_{\nu,\pi}\left[n_{T}(i)\right]\text{KL}(P_i,P'_i),
\end{equation*}
where $\mathbb{P}_{\nu,\pi}$ is the distribution of
$T$-round action-reward histories induced by the interconnection between policy $\pi$ and the environment $\nu$, and $\text{KL}(\mathbb{P}_{\nu,\pi},\mathbb{P}_{\nu',\pi})$ measures the relative entropy between $\mathbb{P}_{\nu,\pi}$ and $\mathbb{P}_{\nu',\pi}$.
In addition, from the high-probability Pinsker inequality, we have
\begin{equation*}
    \text{KL}(\mathbb{P}_{\nu,\pi},\mathbb{P}_{\nu',\pi})\geq \log \frac{1}{2\left(\mathbb{P}_{\nu,\pi}(A)+\mathbb{P}_{\nu',\pi}(A^{c})\right)},
\end{equation*}
where $A$ is any event defined over $\mathbb{P}_{\nu,\pi}$ and $\mathbb{P}_{\nu',\pi}$.
By definition, the regret of policy $\pi$ over $\nu$ and $\nu'$ satisfies
\begin{equation*}
    R_T(\nu,\pi)\geq \frac{T\Delta_i}{2}\mathbb{P}_{\nu,\pi}\left(n_T(i)\geq \frac{T\Theta}{2}\right),
\end{equation*}
\begin{equation*}
    R_T(\nu',\pi)\geq \frac{T(\lambda-\Delta_i)}{2}\mathbb{P}_{\nu',\pi}\left(n_T(i)< \frac{T\Theta}{2}\right).
\end{equation*}
The above equation bases on the fact that the suboptimality gaps in $\nu'$ is larger than $\lambda - \Delta_i$.

Concluding the above two equations and lower bounding $\Delta_i$ and $(\lambda-\Delta_{i})/2$ by $\kappa(\Delta_{i},\lambda):=\min\{\Delta_{i},\lambda-\Delta_{i}\}/2$ yields
\begin{equation*}
\begin{split}
    &\mathbb{P}_{\nu,\pi}\left(n_T(i)\geq \frac{T\Theta}{2}\right)+\mathbb{P}_{\nu',\pi}\left(n_T(i)< \frac{T\Theta}{2}\right)\\
    & \leq \frac{R_T(\nu,\pi)+R_T(\nu',\pi)}{\kappa(\Delta_{i},\lambda)T}.
\end{split}
\end{equation*}

We have
\begin{equation*}
\begin{split}
    &\text{KL}(P_i,P'_i)\mathbb{E}_{\nu,\pi}\left[n_{T}(i)\right] \\ \geq &\log\left(\frac{\kappa(\Delta_{i},\lambda)}{2}\frac{T\Theta}{R_T(\nu,\pi)+R_T(\nu',\pi)}\right)\\
    =&\log(\frac{\kappa(\Delta_{i},\lambda)}{2})+\log(T\Theta)-\log(R_T(\nu,\pi)+R_T(\nu',\pi))\\
    \geq & \log(\frac{\kappa(\Delta_{i},\lambda)}{2})+(1-\sigma)\log(T\Theta)+C,
\end{split}
\end{equation*}
where $C$ is a constant.
The last inequality is based on the assumption that the algorithm is consistent.
Taking $\lambda=\Delta_i$, for large $T$, we can lower bound the regret of any consistent policy $\pi$ as follows:
\begin{equation*}
\begin{split}
    \liminf\limits_{T\rightarrow +\infty} \frac{R_T}{\log(T\Theta)}\geq &\liminf\limits_{T\rightarrow +\infty} \frac{\sum_{i}\mathbb{E}_{\nu,\pi}\left[n_{T}(i)\right]\Delta_i}{\log(T\Theta)}\\
    =&O\left(\sum_i\frac{\Delta_i}{\text{KL}(P_i,P'_i)}\right).
\end{split}
\end{equation*}

This completes the proof.

\subsection{A Proof of Theorem \ref{thm:2}}

In this subsection, we provide a proof for the regret of \ucbo as stated in  Theorem~\ref{thm:2}. In our analysis, we categorize decisions made by the agents into  Type-\uppercase\expandafter{\romannumeral1} and Type-\uppercase\expandafter{\romannumeral2} decisions. Type-\uppercase\expandafter{\romannumeral1}  corresponds to the decisions of an agent when the mean values of local arms lie in the confidence intervals calculated by the agent, otherwise,  Type-\uppercase\expandafter{\romannumeral2} decision happens, i.e., the actual mean value of some local arm is not within the calculated confidence interval.
Specifically, when agent $j$ makes a Type-\uppercase\expandafter{\romannumeral1} decision at time $t$, the following equation holds for any $i$ in $\mathcal{K}_j$.
\begin{equation*}
\begin{split}
\mu(i)
\in  \left[ \hat{\mu}\left(i,\hat{n}_{t}^j(i)\right) - \cint(i,j,t),\hat{\mu}\left(i,\hat{n}_{t}^j(i)\right) + \cint(i,j,t)\right].
\end{split}
\end{equation*}

The following Lemma provides the probability that a Type-\uppercase\expandafter{\romannumeral1} decision happens at a particular decision round.

\begin{lemma}
	\label{lem:1_ac}
At any time slot $t$ when an agent makes its $l$-the decision, it makes a Type-\uppercase\expandafter{\romannumeral1} decision with a probability at least $1-2\sum_{i\in \mathcal{K}_j}\frac{l\Theta_i}{\theta_j} \delta_{l/\theta_j}^{\alpha}$.
\end{lemma}

\begin{proof}
Note that for any arm $i$ with $n$ observations, there is
	\begin{equation*}
		\Pr\left(\mu(i)> \hat{\mu}(i,n)+\sqrt{\frac{\alpha\log \delta^{-1}}{2n}}\right)\leq \delta^{\alpha}.
	\end{equation*}
	Thus, the probability that the true mean value of arm $i$ is above the upper confidence bound in agent $j$ at some time slot $t=l/\theta_j$ is not larger than $\frac{l\theta_j}{\theta_i} \delta_{l\theta_j}^{\alpha}$, which is shown in the following equation.
	\begin{equation*}
		\label{eq:vio_p1}
		\begin{split}
			&\Pr\left(\mu(i)> \hat{\mu}\left(i,\hat{n}^j_{l/\theta_j}(i)\right)+\sqrt{\frac{\alpha\log \delta_{l/\theta_j}^{-1}}{2\hat{n}^j_{l/\theta_j}(i)}}\right)\\
			\leq &\sum_{s=1}^{l\Theta_i/\theta_j }\Pr\left(\mu(i)> \hat{\mu}(i,s)+\sqrt{\frac{\alpha\log \delta_{l/\theta_j}^{-1}}{2s}}\right)
			\leq  \frac{l\Theta_i}{\theta_j} \delta_{l/\theta_j}^{\alpha}.
		\end{split}
	\end{equation*}
	Similarly, we have
	\begin{equation*}
		\Pr\left(\mu(i)< \hat{\mu}(i,\hat{n}^j_{l/\theta_j}(i))-\sqrt{\frac{\alpha\log \delta_{l/\theta_j}^{-1}}{2\hat{n}^j_{l/\theta_j}(i)}}\right)\leq \frac{l\Theta_i}{\theta_j} \delta_{l/\theta_j}^{\alpha}.
	\end{equation*}
	
	Thus, the probability that the mean value of any arm in $\mathcal{K}_j$ at time $t$ lies in the its confidence interval is lower bounded by
		$1-2\sum_{i\in \mathcal{K}_j}\left(l\Theta_i \delta_{l/\theta_j}^{\alpha}\right)/\theta_j$.
	This completes the proof.
\end{proof}

\begin{lemma}\label{lem:2_ac}
If at any time $t\leq T$ agent $j\in \mathcal{A}_{-i}^*$ makes a Type-\uppercase\expandafter{\romannumeral1} decision and pulls arm $i$, i.e., $I_t^j=i$, we have
   	\begin{equation*}
		\hat{n}_t^j(i)\leq  \frac{2\alpha\log \delta^{-1}}{\Delta^2(i_j^*,i)}.
	\end{equation*}
\end{lemma}
\begin{proof}
Consider that agent $j$, $j\in \mathcal{A}^*_{-i}$ makes a Type-\uppercase\expandafter{\romannumeral1} decision at time slot $t$ and $I_t^j=i$. We have that the following equation holds.
	\begin{equation} \label{eq:con_1}
		\cint(i,j,t)\geq \Delta(i_j^*,i).
	\end{equation}
	
	Otherwise, we have
	\begin{equation*}
		\begin{split}
			\hat{\mu}(i_j^*,\hat{n}^j_{t}(i_j^*))+\cint(i_j^*,j,t)
			\geq & \mu(i_j^*)=\mu(i)+\Delta(i_j^*,i)\\
			> & \mu(i)+\cint(i,j,t),
		\end{split}
	\end{equation*}
	contradicting the fact that $I_t^j=i$.
    Combining Equation (\ref{eq:con_1}) and the definition of $\cint(i,j,t)$, we have
   	\begin{equation*}
		\hat{n}_t^j(i)\leq  \frac{2\alpha\log \delta^{-1}}{\Delta^2(i_j^*,i)}.
	\end{equation*}

   This completes the proof.
	\end{proof}

Based on lemmas~\ref{lem:1_ac} and~\ref{lem:2_ac}, we proceed to prove Theorem~\ref{thm:2}.
Recall that $\mathcal{A}^*_{-i} = \{j_m:m=1,2,\ldots,M_{-i}\}$ is the set of agent including arm $i$ as a suboptimal arm with $M_{-i} = |\mathcal{A}^*_{-i}|$ as the number of such agents.
Further and without loss of generality, we assume $\mu(i^*_{j_1})\geq \mu(i^*_{j_2})\geq\ldots\geq \mu(i^*_{j_{M_{-i}}})$.

We assume an agent makes a Type-\uppercase\expandafter{\romannumeral1} decision.
From Lemma \ref{lem:2_ac}, when $i$ is selected by an agent in $\mathcal{A}^*_{-i}$, the total number of selection times by agents in $\mathcal{A}^*_{-i}$ for suboptimal arm $i$ is upper bounded by
\begin{equation}
\label{eq:n_sub}
\frac{2\alpha\log \delta^{-1}}{\Delta^2(i_{j_{M_{-i}}}^*,i)}+\sum_{j\in \mathcal{A}^*_{-i}}\min\left\{d_j\theta_j,n^j_{t}(i)\right\},
\end{equation}
where the second term in the above equation serves as an upper bound of the number of outstanding observations by agents in $\mathcal{A}^*_{-i}$ on arm $i$ due to delays. By applying the result for the basic \AAEbasic algorithm, we have
\begin{equation} \label{eq:n_ub}
   \E\left[ n^j_{t}(i)\right]\leq \frac{2\alpha\log \delta^{-1}}{\Delta^2(i_{j}^*,i)}, ~j\in \mathcal{A}^*_{-i}.
\end{equation}

Let $Q_1$ be the number of Type-\uppercase\expandafter{\romannumeral2} decisions for all agents.
As a result of Lemma \ref{lem:1_ac}, we have
	\begin{equation}
	\label{eq:typeii}
		 \E\left[Q_1\right]\leq  2\sum_{j\in \mathcal{A}}\sum_{l=1}^{N_j}\sum_{i\in \mathcal{K}_j}\frac{l\Theta_i}{\theta_j} \delta_{l/\theta_j}^{\alpha}=q_1.
	\end{equation}
Combining Equations~\eqref{eq:n_sub}, \eqref{eq:n_ub} and~\eqref{eq:typeii}, we can build up an upper bound for the total expected number of selection arm $i$ by agents in $\mathcal{A}^*_{-i}$ as
\begin{equation*}
\begin{split}
&\E\left[\sum_{m=1}^{M_{-i}}n^{j_m}_{T}(i)\right] \\
\leq & \E\left[Q_1\right]+\frac{2\alpha\log \delta^{-1}}{\Delta^2(i_{j_{M_{-i}}}^*,i)}+\sum_{j\in \mathcal{A}^*_{-i}}\min\left\{d_j\theta_j,\frac{2\alpha\log \delta^{-1}}{\Delta^2(i_{j}^*,i)}\right\}.
\end{split}
\end{equation*}

By similar reasoning, we upper bound the expected number of selection times by the first $m$ agents, i.e., agents  $j_1,j_2,\ldots,j_m$, by
\begin{equation*}
 \E\left[Q_1\right]+\frac{2\alpha\log \delta^{-1}}{\Delta^2(i^*_{j_m},i)}+\sum_{k=1}^{m}\min\left\{d_{j_k}\theta_{j_k},\frac{2\alpha\log \delta^{-1}}{\Delta^2(i_{j_k}^*,i)}\right\}.
\end{equation*}

For simplicity of analysis, we define $A_m$ and $B_m$ as
\begin{equation*}
    A_m:=\frac{2\alpha\log \delta^{-1}}{\Delta^2(i^*_{j_m},i)},
\end{equation*}

\begin{equation*}
    B_m:=\sum_{k=1}^{m}\min\left\{d_{j_k}\theta_{j_k},\frac{2\alpha\log \delta^{-1}}{\Delta^2(i_{j_k}^*,i)}\right\}.
\end{equation*}

With the above facts, we can upper bound the regret spent on arm $i$ as in Equation (\ref{eq:sum_regret}).
\begin{figure*}
\begin{equation}\label{eq:sum_regret}
\begin{split}
&\E\left[\sum_{j\in \mathcal{A}^*_{-i}}n^j_{T}(i)\Delta(i^*_{j},i)\right] =\E\left[\sum_{m=1}^{M_{-i}}n^{j_m}_{T}(i)\Delta(i^*_{j},i)\right]\\
& \leq \left(\E\left[Q_1\right]+A_1+B_1+1\right)\Delta(i^*_{j_1},i) +\sum_{m=1}^{M_{-i}-1} (A_{m+1}+B_{m+1}-A_m-B_m)\Delta(i^*_{j_{m+1}},i) \\
&= \left(\E\left[Q_1\right]+A_1\right)\Delta(i^*_{j_1},i)+\sum_{m=1}^{M_{-i}-1} \left(A_{m+1}-A_m\right)\Delta(i^*_{j_{m+1}},i)+ \sum_{m=1}^{M_{-i}}\min\left\{d_{j_m}\theta_{j_m},\frac{2\alpha\log \delta^{-1}}{\Delta(i_{j_m}^*,i)}\right\} \\
& \leq  \left(\E\left[Q_1\right]+A_1\right)\Delta(i^*_{j_1},i) +\sum_{m=1}^{M_{-i}-1} A_m(\Delta(i^*_{j_{m}},i)-\Delta(i^*_{j_{m+1}},i))
+ A_{M_{-i}} \Delta(i^*_{j_{M_{-i}}},i)+ f_i(\delta)\\
& \leq \left(\E\left[Q_1\right]+A_1+1\right)\Delta(i^*_{j_1},i)+\frac{8\alpha\log \delta^{-1}}{\Delta(i^*_{j_1},i)}+ \int_{\Delta(i^*_{j_{M_{-i}}},i)}^{\Delta(i^*_{j_1},i)} \frac{2\alpha\log \delta^{-1}}{z^2} dz+\frac{2\alpha\log \delta^{-1}}{\Delta(i^*_{j_{M_{-i}}},i)} + f_i(\delta)\\
& = \left(\E\left[Q_1\right]+A_1\right)\Delta(i^*_{j_1},i)+\frac{2\alpha\log \delta^{-1}}{\Delta(i^*_{j_1},i)}+ \frac{2\alpha\log \delta^{-1}}{\Delta(i^*_{j_{M_{-i}}},i)}-\frac{2\alpha\log \delta^{-1}}{\Delta(i^*_{j_1},i)} +\frac{2\alpha\log \delta^{-1}}{\Delta(i^*_{j_{M_{-i}}},i)} + f_i(\delta)\\
& = \left(\E\left[Q_1\right]+A_1\right)\Delta(i^*_{j_1},i)+\frac{4\alpha\log \delta^{-1}}{\Delta(i^*_{j_{M_{-i}}},i)}+ f_i(\delta)
\leq \E\left[Q_1\right]+\frac{6\alpha\log \delta^{-1}}{\Delta(i^*_{j_{M_{-i}}},i)}+ f_i(\delta).
\end{split}
\end{equation}
\end{figure*}
In the derivation, we use an Abel transformation in the second equality.
According to Equation~\eqref{eq:tilde_delta}, we have $\tilde{\Delta}_i=\Delta(i^*_{j_{M_{-i}}},i)$. Thus, we have
\begin{equation*}
\begin{split}
\E\left[R_T\right]
\leq &\sum_{i:\tilde{\Delta}_i>0}\E\left[\sum_{j\in \mathcal{A}^*_{-i}}n^j_{T}(i)\Delta(i^*_{j},i)\right] \\
\leq &\sum_{i:\tilde{\Delta}_i>0}\left(\frac{6\alpha\log \delta^{-1}}{\tilde{\Delta}_i}+\E\left[Q_1\right]+ f_i(\delta)\right)\\
\leq &\sum_{i:\tilde{\Delta}_i>0}\left(\frac{6\alpha\log \delta^{-1}}{\tilde{\Delta}_i}+q_1+f_i(\delta)\right) .
\end{split}
\end{equation*}

This completes the proof of Theorem \ref{thm:2}.

\subsection{A Proof of Theorem \ref{thm:4}}
The proof of Theorem \ref{thm:4} also leverages the notions of Type-\uppercase\expandafter{\romannumeral1}/Type-\uppercase\expandafter{\romannumeral2} decisions.
Specifically, with Type-\uppercase\expandafter{\romannumeral1} decisions, an agent is able to keep the local optimal arm in its candidate set and eventually converges its decisions to the local optimal arm.
Similarly, we have the following lemma.

\begin{lemma}\label{lem:2_ac}
If at any time $t\leq T$ agent $j\in \mathcal{A}_{-i}^*$ by \AAE makes a Type-\uppercase\expandafter{\romannumeral1} decision and pulls arm $i$, i.e., $I_t^j=i$, we have
   	\begin{equation*}
		\hat{n}_t^j(i)\leq  \frac{8\alpha\log \delta^{-1}}{\Delta^2(i_j^*,i)}+1.
	\end{equation*}
\end{lemma}

\begin{proof}
We consider agent $j\in \mathcal{A}^*_{-i}$ running \AAE makes a Type-\uppercase\expandafter{\romannumeral1} decision at time $t$ and $I_t^j=i$. First, we claim that the following holds.
	\begin{equation} \label{eq:con_2}
		2\cint\left(i_j^*,j,t\right)+2\cint\left(i,j,t\right)\geq\Delta(i_j^*,i).
	\end{equation}
	
	Otherwise, we have
	\begin{equation*}
	\begin{split}
			&\hat{\mu}\left(i^*_j,\hat{n}^j_{t}(i_j^*)\right)-\cint\left(i_j^*,j,t\right)\\
			=&\hat{\mu}\left(i_j^*,\hat{n}^j_{t}(i_j^*)\right)+\cint\left(i_j^*,j,t\right)-2\cint\left(i_j^*,j,t\right)\\
			\geq&\mu\left(i_j^*\right)-2\cint\left(i_j^*,j,t\right)
			= \mu(i)+\Delta\left(i_j^*,i\right)-2\cint\left(i_j^*,j,t\right)\\
			> &\mu(i)+2\cint\left(i,j,t\right)\geq \hat{\mu}\left(i,\hat{n}^j_{t}(i)\right)+\cint\left(i,t\right).
	\end{split}
	\end{equation*}
	It shows the fact that the lower bound of arm $i^*_j$ is larger than the upper bound of arm $i$,
	contradicting the fact that $I_t^j=i$.

We continue the rest proof by considering the following two cases.

(1) We consider the case where $\hat{n}^j_t(i)<\hat{n}^j_t(i_j^*)$ and $I_t^j=i$.
	It follows from Equation (\ref{eq:con_2}) that $4\cint\left(i,j,t\right)\geq \Delta(i_j^*,i)$.
	Thus, in this case, the number of observations of $i$ received by agent $j$ is upper bounded by
	\begin{equation*}
		\hat{n}^j_t(i)\leq \sup_{t}\frac{8\alpha\log \delta_t^{-1}}{\Delta^2(i_j^*,i)}.
	\end{equation*}
	
(2) We proceed to consider the case where $\hat{n}^j_t(i)\geq \hat{n}^j_t(i_j^*)$ and $I_t^j=i$.
	We have	$4\cint\left(i_j^*,j,t\right)\geq \Delta(i_j^*,i)$.
	Thus, the number of observations on $i_j^*$ by agent $j$ is upper bounded as below.
	\begin{equation*}
		\hat{n}_t^j(i_j^*)\leq \sup_t \frac{8\alpha\log \delta_t^{-1}}{\Delta^2(i_j^*,i)}.
	\end{equation*}
	Note that, if $\hat{n}^j_t(i)>\hat{n}_t^j(i_j^*)+1$, the agent $j$ will not select $i$, since in \AAE, the agent selects the arm with the least observations.
	Hence, the number of selection for $i$ by agent $j$ is not larger than $\hat{n}_t^j( i^*_j)+1$, and we have
	\begin{equation*}
		\hat{n}_t^j(i)\leq \sup_t \frac{8\alpha\log \delta_t^{-1}}{\Delta^2(i_j^*,i)}+1\leq  \frac{8\alpha\log \delta^{-1}}{\Delta^2(i_j^*,i)}+1.
	\end{equation*}

   Concluding case (1) and (2) completes the proof.

\end{proof}

We skip the rest of the proof for Theorem \ref{thm:4}, since it follows similar lines to that of Theorem \ref{thm:2}.

\subsection{A proof of Theorem \ref{thm:comlexity2}}

Generally, the proof contains two steps. The first one is to upper bound the number of sent messages on sup-optimal arms, and the second one is to upper bound that on the local optimal arms.

(1) We assume at time $t$, an agent in $\mathcal{A}^*_{-i}$ makes a Type-\uppercase\expandafter{\romannumeral1} decision to select arm $i$.
From Lemma \ref{lem:2_ac}, we can upper bound the total number of selection times by agents in $\mathcal{A}_{i}$ for suboptimal arm $i$ up to $t$ by
\begin{equation*}
    \frac{8\alpha\log \delta^{-1}}{\Delta^2(i_{j_{M_{-i}}}^*,i)}+\sum_{j\in \mathcal{A}^*_{-i}}d_j\theta_j+1,
\end{equation*}
where the second term corresponds to an upper bound for the number of outstanding observations on arm $i$.

Combined with the fact that the expected number of Type-\uppercase\expandafter{\romannumeral2} decisions for all agents is upper bounded by $q$, we can upper bound the expected number of observations on arm $i$ by agents in $\mathcal{A}^*_{-i}$ as follows.

	\begin{equation*}
	\begin{split}
		&\frac{8\alpha\log \delta^{-1}}{\Delta^2(i_{j_{M_{-i}}}^*,i)}+\sum_{j\in \mathcal{A}^*_{-i}}d_j\theta_j+\frac{2}{\alpha-2}\sum_{j\in \mathcal{A}} \Theta\theta_j^{\alpha-1}+1\\
		=&\frac{8\alpha\log \delta^{-1}}{\Delta^2(i_{j_{M_{-i}}}^*,i)}+\sum_{j\in \mathcal{A}^*_{-i}}d_j\theta_j+q_2+1.
	\end{split}
	\end{equation*}

Accordingly, the expected number of messages sent by $\mathcal{A}^*_{-i}$ to broadcast those observations on arm $i$ is upper bounded by
	\begin{equation*}
		\left(\frac{8\alpha\log \delta^{-1}}{\Delta^2(i_{j_{M_{-i}}}^*,i)}+\sum_{j\in \mathcal{A}^*_{-i}}d_j\theta_j+q_2+1\right)M_i.
	\end{equation*}

Then we can further upper bound the total number of messages sent by agents for broadcasting the observations of their suboptimal arms by
	\begin{equation*}
		\sum_{i\in \mathcal{K}}\left(\frac{8\alpha\log \delta^{-1}}{\Delta^2(i_{j_{M_{-i}}}^*,i)}+\sum_{j\in \mathcal{A}^*_{-i}}d_j\theta_j+q_2+1\right)M_i.
	\end{equation*}

(2) By the rules of the \AAE algorithm, we have that an agent broadcasts its observations only when its candidate set has more than one arms. That is, the number of rounds where an agent broadcasts observations on the optimal arm is not larger than the number of pulling a suboptimal arm. Then, we have the following upper bound for the number of rounds where an agent broadcasts observations on the optimal arm.
	\begin{equation*}
		\sum_{i\in \mathcal{K}}\left(\frac{8\alpha\log \delta^{-1}}{\Delta^2(i_{j_{M_{-i}}}^*,i)}+\sum_{j\in \mathcal{A}^*_{-i}}d_j\theta_j+q_2+1\right).
	\end{equation*}

Hence, the number of messages on the local optimal arms is upper bounded by
	\begin{equation*}
		\sum_{i\in \mathcal{K}}\left(\frac{8\alpha\log \delta^{-1}}{\Delta^2(i_{j_{M_{-i}}}^*,i)}+\sum_{j\in \mathcal{A}^*_{-i}}d_j\theta_j+q_2+1\right)M.
	\end{equation*}

Combining the above two cases yields an upper bound on the expected number of messages sent by the agents, which is
	\begin{equation*}
	\begin{split}
		&\sum_{i\in \mathcal{K}}\left(\frac{8\alpha\log \delta^{-1}}{\Delta^2(i_{j_{M_{-i}}}^*,i)}+\sum_{j\in \mathcal{A}^*_{-i}}d_j\theta_j+q_2+1\right)(M+M_i) \\
		=&\sum_{i\in \mathcal{K}}\left(\frac{8\alpha\log \delta^{-1}}{\tilde{\Delta}^2_i}+\sum_{j\in \mathcal{A}^*_{-i}}d_j\theta_j+q_2+1\right)(M+M_i) \\
			=&\sum_{i\in \mathcal{K}}\left(\frac{8\alpha\log T}{\tilde{\Delta}^2_i}+\sum_{j\in \mathcal{A}^*_{-i}}d_j\theta_j+q_2+1\right)(M+M_i) .
	\end{split}
	\end{equation*}

This completes the proof.

\end{document}